\documentclass[a4paper]{article}
\pdfoutput=1
\usepackage[T1]{fontenc}
\usepackage[utf8]{inputenc}
\usepackage{textcomp}
\usepackage{parskip}
\usepackage[hidelinks]{hyperref}
\newcommand*{\email}[1]{\href{mailto:#1}{\nolinkurl{#1}}} 
\usepackage{makeidx}
\usepackage{graphicx}
\usepackage{rotating}
\usepackage{float}
\usepackage{changepage}
\usepackage{algorithm, setspace}
\usepackage{algorithmic, setspace}
\usepackage[toctitles]{titlesec}
\usepackage{fancyhdr}
\usepackage{hyperref}
\usepackage{amssymb}
\usepackage{amsmath,bm}
\usepackage{latexsym}
\usepackage{amsthm}
\usepackage{amssymb}
\usepackage{fancyvrb}
\usepackage{sectsty}
\usepackage{fancybox}
\usepackage{listings}
\usepackage{charter}
\usepackage{geometry}
\usepackage{tikz}
\usetikzlibrary{topaths,calc}
\usepackage{enumitem}
\usepackage[center]{caption}
\usepackage{epsfig}
\usepackage{tabularx}
\usepackage{latexsym}
\usepackage{alltt}
\usepackage{setspace}
\usepackage{graphicx,wrapfig,lipsum}
\usepackage{datetime}
\newdateformat{mydate}{\monthname[\THEMONTH], \THEYEAR}
\usepackage{pslatex}
\usepackage{times}

\newtheorem{definition}{Definition~}

\newtheorem{lemma}{Lemma~}
\newtheorem{notation}{Notation~}
\def\abstract{{\begin{center}
\Large {\bf Abstract}
\end{center} }}

\bmdefine\bchi{\bm\chi}
\DefineVerbatimEnvironment{code}{Verbatim}{fontsize=\small}
\DefineVerbatimEnvironment{example}{Verbatim}{fontsize=\small}
\title{BINARY SEARCH AND FIRST ORDER GRADIENT BASED METHOD FOR STOCHASTIC OPTIMIZATION }
\date{}
\author{ Vijay Pandey\thanks{Completed master's degree in Computer Science and Engineering, from Indian Institute of Technology Kharagpur, India}\\ \email{vijayiitkgp13@gmail.com} 
}

\begin{document}
\maketitle
\begin{abstract}
\begin{adjustwidth}{1.4cm}{1.4cm}
In this paper, we present a novel stochastic optimization method, which uses the binary search technique with first order gradient based optimization method, called Binary Search Gradient Optimization (BSG) or \textit{BiGrad}. In this optimization setup, a non-convex surface is treated as a set of convex surfaces. In BSG, at first, a region is defined, assuming region is convex. If region is not convex, then the algorithm leaves the region very fast and defines a new one, otherwise, it tries to converge at the optimal point of the region. In BSG, core purpose of binary search is to decide, whether region is convex or not in logarithmic time, whereas, first order gradient based method is primarily applied, to define a new region. In this paper, Adam is used as a first order gradient based method, nevertheless, other methods of this class may also be considered. In deep neural network setup, it handles the problem of vanishing and exploding gradient efficiently. We evaluate BSG on the MNIST handwritten digit, IMDB, and CIFAR10 data set, using logistic regression and deep neural networks. We produce more promising results as compared to other first order gradient based optimization methods. Furthermore, proposed algorithm generalizes significantly better on unseen data as compared to other methods.
\end{adjustwidth}
\end{abstract}
\section{Introduction}
Convex optimization is a sub-field of mathematical optimization that studies the problem of minimizing convex functions over convex sets. In general, no analytical formula for the solution of convex optimization problems is present, however there are very effective methods available to solve them. It is reasonable to expect that solving general convex optimization problems will become a technology within a few years. Moreover, to solve convex optimization problem, class of  first order iterative optimization algorithm has been proposed. Gradient descent (GD) is one of them. GD finds local minimum of a differentiable function. It has the application in large scale  optimization \cite{bottou2018optimization}. Originally, GD was known as a convex optimization algorithm. Nevertheless in recent years it gained significant focus as a non-convex optimization algorithm as well, with increased popularity of deep neural networks, as the cost function of deep neural networks are non-convex in nature. With increasing popularity of deep neural networks and demand of AI solutions, significant amount of researches are being carried out to accelerate the progress in the the field of non-convex optimization. There are various optimization methods which are extensively used in practice, showing the remarkable performance. Few of these are the variants of GD which includes but not limited to, stochastic gradient descent (SGD) \cite{robbins1951stochastic}, adaptive learning methods, momentum methods \cite{sutskever2013importance,momentum_term}, RMSProp \cite{rmsprop_lecture}, and nesterov  methods. SGD is widely used in deep learning and classical machine learning for cost function optimization, with producing excellent results. Afterwards, discovery of momentum based and nesterov based optimization methods enhanced the optimization performance. Adaptive learning methods came into existence where learning rate is managed by algorithm itself in accordance with the training. Adam is based on adaptive learning, and it showed significant improvement in this direction \cite{kingma2014Adam}. Adding nesterov method in Adam showed better performance on some set of problems \cite{Dozat2016IncorporatingNM}. Various other approaches, incorporating the adaptive learning methods came into practice \cite{adagrad,zeiler2012adadelta,shazeer2018adafactor,ward2018adagrad,reddi2019convergence, rmsprop, zhou2018convergence}. To get more out of these methods, better estimates of hyperparameters such as learning rate, decay constant, momentum coefficient, etc., play crucial role in terms of faster and better convergence. One major challenge in non-convex optimization is handling of saddle point, the point where for some dimensions it is minima while for some other it is maxima. More formally, condition for existence of saddle point is, if eigen values of hessian matrix of loss function corresponding to its parameters, include some positive values and some negative values. Lot of research have been carried out in overcoming this problem \cite{jin2017escape, fang2019sharp, levy2016power, daneshmand2018escaping, dauphin2014identifying, ge2015escaping, reddi2017generic, mokhtari2018escaping, daneshm2018escaping, xu2017firstorder}.
\paragraph{}
Proposed novel optimization method BSG shows state of the art convergence speed as compared with other contemporary algorithms, and it achieves better generalization on unseen data as well. BSG is influenced by SGD and binary search technique. In SGD principal focus is to avoid the overshooting of the target i.e., local minima while training. However, in proposed approach, problem of overshooting is not of main concern. In BSG, sign of gradient plays much crucial role as compared to the gradient magnitude. Moreover, it also overcomes the vanishing gradient and exploding gradient problem at some extent.

\section{Binary Search Gradient Optimization}
\subsection{Preliminaries}
\begin{notation}\label{notation}
Here we define some notations which will be used in remaining part of the paper. $x_{t+i}$: value, variable $x$ holds at time $t+i$; $L$: Loss function; $g(x):L'(x)$; $g^+(x):g(x)>0$; $g^-(x):g(x)<0$; $g^0(x):g(x)=0$; $A$: first order gradient based optimization method (Adam is used as $A$ in this paper); $w$: weight; $n$: negative gradient boundary ($g^-(n)$); $p$: positive gradient boundary ($g^+(p)$); $u$: calculated update value by $A$ ($w_{t+1}=w_t - u_t$); $S:[n,p]$ is the interval to search for $w$.
\end{notation}
\begin{definition}\label{convex}
A function $f:\Re^d \rightarrow$ is convex if for all $x, y \in \Re^d$, for all $\lambda \in [0,1]$,
\begin{center}
    $\lambda f(x)+(1-\lambda)f(y) \geq f(\lambda x + (1-\lambda)y)$
\end{center}
\end{definition}
\begin{definition}\label{optimum}
If a function $f:\Re^d \rightarrow$ is convex, for $x\in \Re^d$, if $f'(x)=0$, then at $x$ local minima of $f$ exists.
\end{definition}
\begin{definition}\label{convex_zero}
In strictly convex function $f:\Re^d \rightarrow$, if for all $n, p, w\in \Re^d$,  $f'(n)<0$ and $f'(p)>0$, then there must exist a point $w$, where $f'(w)=0$ and which satisfies,
\begin{center}
$f'(n) < f'(w) < f'(p)$.
\end{center}
\end{definition}
\begin{definition}\label{convex_sorted}
In strictly convex function $f:\Re^d \rightarrow$, if for all $n, p, w\in \Re^d$, $f'(n)<0$ and $f'(p)>0$, then all values from $f'(n)$ to $f'(p)$ and $n$ to $p$ are in ascending order.
\end{definition}

\begin{definition}\label{bs}
\textbf{Binary Search}\cite{binary_search}: Given an array $A$ of $n$ elements with values or records $A_{0},A_{1},A_{2},\ldots ,A_{n-1}$ in sorted order such that $ A_{0}\leq A_{1}\leq A_{2}\leq \cdots \leq A_{n-1}$, and target value $T$. Binary search finds $T$ by comparing it with an element in the middle of the array. If $T$ matches the element, its position in the array is returned. If $T$ is less than the element, the search continues in the lower half of the array. If $T$ is greater than the element, the search continues in the upper half of the array. This iterative procedure keeps track of the search boundaries and returns the index of $T$ if found. Time complexity of Binary Search is $ O(Log(n))$.
\end{definition}
\begin{lemma}\label{principal}
In strictly convex function $f:\Re^d \rightarrow$, for $n, p, w\in \Re^d$, such that $f'(n)<0$, $f'(p)>0$, and $|[f'(n),f'(p)]|=t$ then there exists a $w$, which satisfies $f'(w)=0$, and can be found in $O(Log(t))$ time.
\end{lemma}
\begin{proof}
We use binary search as per definition \ref{bs}, in proving this lemma. Let assume there is an array $A$ of cardinality $|A|=t$, where $A[i]=f'(i)$, having $f'(n)$ as first element and $f'(p)$ as last element. As per definition \ref{convex_sorted}, $A$ will be in ascending order. Here $f'(w)=0$ and $T=f'(w)$, therefore as per definition \ref{bs}, $w$ can be found in $O(Log(t))$ time.
\end{proof}
\subsection{Algorithm}

\begin{algorithm}[H]
\caption{Computing BSG update at time $t$}
\begin{algorithmic}[1]\label{algo}
\setstretch{0.02}
\REQUIRE Interval factor $\alpha$
\REQUIRE Initial parameter $x_1$
\STATE Initialize negative gradient boundary $n_0\leftarrow 100$
\STATE Initialize positive gradient boundary $p_0 \leftarrow 0$
\STATE Initialize timestep $t\leftarrow 0$
\WHILE{$x_t$ not converged}
\STATE $t\leftarrow t+1$
\STATE Compute gradient $g_t$
\STATE Compute $u_t$ using $A$
\STATE $s \leftarrow n_{t-1} - p_{t-1} + |u_t|$
\IF{$s>0$} 
\STATE $r \leftarrow 1$
\ELSE
\STATE $r \leftarrow 0$
\ENDIF
\IF{$g_t\leq 0$}
\STATE $n_t \leftarrow x_t-u_t$ 
\STATE $p_t \leftarrow p_{t-1}*(1-r) + (n_t -\alpha* u_t)*r$
\ELSE
\STATE $p_t \leftarrow x_t-u_t$
\STATE $n_t \leftarrow n_{t-1}*(1-r) + (p_t -\alpha* u_t)*r$ 
\ENDIF
\STATE $x_{t+1} \leftarrow \frac{n_t+p_t}{2}$
\ENDWHILE
\end{algorithmic}
\end{algorithm}

\subsection{Intuitive Idea}
In this subsection, we will understand the core idea behind the proposed method. In classical machine learning, surface of cost function is convex in nature. However in deep learning, surface of cost function is very complex and non-convex in nature. First understand the working of BSG from the convex surface perspective, then generalize it for non-convex surfaces. As per definition \ref{optimum}, in convex setting, given that $f:\Re^d \rightarrow$ is a convex function, it is desired to find a point $w$, where $f'(w)=0$. Lemma \ref{principal} suggests that point $w$ can be found in logarithmic time. Major challenge is to know two points, $n, p\in \Re^d$ beforehand, which satisfies $f'(n)<0$ and $f'(p)>0$. If we are provided with $n$ and $p$ at first, then task of finding $w$ seems much easier. Intuitively, proposed method can be thought as treating the error surface as a set of convex surfaces. First define an area assuming it is convex, and search for optimal point in that. $n_t>p_t$ is the condition when defined region is not convex. In such a scenario, define other area using $n_t, p_t, u_t,$ and interval factor $\alpha$, adjacent to the current one, and again follow the same process. Keep iterating the process, until defined region is convex. Once convex area is found, algorithm starts converging in that area to find the local minima. BSG choose two points $n$ and $p$, which satisfies $n<p$, and assumes $f'(n)<0$ and $f'(p)>0$. If assumption falls right, then $f'(w) \in [f'(n), f'(p)]$, otherwise, new $n$ and $p$ will be selected. This process iterates until points $n$ and $p$, are found, where $f'(n)<0$ and $f'(p)>0$.
\subsection{Working of Algorithm}
In previous subsection, we have discussed how binary search is applied in convex optimization to reach the local minima. In this subsection, we will discuss the significance of updated value $u$, produced by $A$ in great detail. After subtracting $u_t$ with $n_t$ or $p_t$, size of $S_t$ contracts faster in right direction. Major significance of subtracting $u$ by either $n$ and $p$ is, it can be known, whether $g^0(x)$ exists in $S$ or not. To understand it better, consider $u_t$ is not introduced, then $S_t$ will only be shrunk in each further iteration, because it assumes that $n_t$ and $p_t$ satisfies $g_t^-(x)$ and $g_t^+(x)$ respectively, which is not true. This behaviour results in constraining the optimization algorithm, to find the parameter only in the limited space, where $g_t^0(x)$ may not exist. Incorporating $u$, facilitates in ensuring whether $g^0(x) \in S$ or $g^0(x) \not{\in} S$. If $g^0(x) \in S$, $u$ will eventually attain very small value after few iterations, or else, $u$ will be still large and will cause the algorithm to reach a point, where $p<n$. This condition reflects that $g^0(x) \not{\in} S$, and as a result, new $S$ will be defined accordingly. $|S|$ depends on $\alpha$ and $u$. Small value of $u_t$ reveals, $g_t^0(x)$ is near to $w_t$, and which contributes in small $|S_t|$, conversely large value of $u_t$ suggests, $g^0(x)_t$ is far situated to $w_t$, and in accordance, large $|S_t|$ is obtained. Small value of $u_t$ also suggests that, $w_t$ is in the flat region, where convergence stops or becomes very slow. It may be the local maxima as well. However, BSG does not stop even at $g_t^0(w)$. It always contracts the $S_t$, assuming $g_t^+(p)$ and $g_t^-(n)$, and tends to find a point $g_t(w) \approx 0$ but not $g_t(w) = 0$ . By this way, it neither stops at saddle point nor at flat region.
\paragraph{}
Now consider, distinct scenarios where formation of new $S$ takes place. Given, $sgn(u_t)=-1$, $n_{t+1}=w_{t+1}-u_{t+1}$ and $p_{t+1}=p_t$. $w_{t+2}=(n_{t+1}+p_{t+1})/2$. For $g_{t+2}^+(w)$, assign $p_{t+2}=w_{t+2} - u_{t+2}, n_{t+2}=n_{t+1}$, and for $g_{t+1}^-(w)$, update $n_{t+2}=w_{t+2}, p_{t+2}=p_{t+1}$. Here case may arise, where $n_{t+2}>p_{t+2}$. This situation occurs in two scenarios. One is, when $w_{t+1}$ lands on the negative slope of another convex region $C \not{\in} S$. This is the case where algorithm starts finding the local maxima. Second scenario is when either $sgn(g'(p_{t+2})) \neq 1 $ or $sgn(g'(n_{t+2})) \neq -1$. Both of the above situations are not desirable. Therefore, to avoid this, we ensure that $n < p$ must always holds, and it is implemented by creating a new $S$ appropriately, whenever above condition is encountered. New $S_t$ is created based on the sign of $g(w_t)$. For $g^-(w_t)$ or $g^0(w_t)$, add $\alpha*|u_t|$ to $w_t$, and assign it to $p_{t}$, otherwise, subtract $\alpha*|u_t|$ to $w_t$ and populate it to $n_{t}$.

\subsection{Salient Features}

In this section, we will discuss the salient aspects of the proposed approach. BSG is highly stochastic in nature, therefore, smaller batch sizes add more stochastic behaviour in it. Very large batch size slows down the learning. BSG works well on average batch size. Average batch size results in, better finding of convex region in non-convex loss surface. However, irrespective of the batch size, algorithm converges in the end. BSG works fairly well with batch normalization as well.
\subsubsection{Handles Vanishing and Exploding Gradient Problem}
It handles the vanishing and exploding gradient problem in deep neural networks. In BSG, once convex region is found, weight update much depends on binary search method, which does consider gradient sign rather than gradient magnitude. Vanishing gradient causes smaller value of $u$, which suggests that descent speed becomes slow. BSG tends to attain the gradient $g(w) \approx 0$. Therefore, slow descent motion is accelerated towards bottom of the valley. Whereas, exploding gradient causes larger value of $u$, results in risk of overshooting. BSG pulls back the weight to achieve $g(w) \approx 0$. From above analysis, it can be seen that, BSG handles vanishing and exploding gradient issue in much efficient way.
\subsubsection{Convergence}
One notable thing with proposed approach is, it does not stop at $g_t(x)=0$. It tries to achieve $g_t(x) \approx 0$, but never converges at $g_t(x)=0$. It helps in avoiding the saddle point. BSG converges, when it is in a valley surrounded by hills having negative slope and positive slope. One notable result is, it finds the minima, where $g_t^+(p)$, $g_t^-(p)$ are met, and value of $u_t$ is infinitesimally small. In this case, there is no more movement for parameters. Proposed optimization method, helps in finding the better trajectory by imposing more than one condition. Desired minima is not met, until above mentioned convergence conditions are satisfied. Parameters resulting using BSG are more robust, and it results in more promising and better minima. 

\subsubsection{Provides Regularization}
One of the advantage with current approach is, adding noise (perturbation) in the gradient, which helps to achieve better generalization \cite{gradient_noise}. In BSG, weight update is result of division of $p+n$ by 2 which indirectly perturbs the gradient, because, convergence step does not only depend on the gradient magnitude. However, adding noise sometimes lead to slower training, nevertheless, in case of BSG, training speed is not get affected. Regularization effect added by BSG helps model to achieve better generalization on test data. 

\subsection{Effect of Interval Factor on Learning}
Several factors are needed to consider, while setting interval factor $\alpha$. Let there are $w_t$, $u_t$, and $g_t^-(w)$. Since, $n_t = w_t - u_t$ and $p_t = p_{t-1}$, therefore, $w_{t+1}=(w_t- u_t + p_t)/2$. $\delta w = w_{t+1}-w_t$, $ \delta w = (p_t-w_t-u_t)/2$, $ \delta w=((p_t -n_t) -2*u_t)/2$. With $\delta w$ update, there may be a case with $w_{t+1}$ that, it may skip the current valley and lands in another valley. Nonetheless, this event is rare and can happen, if $|p_t - n_t|$ is very high. Nevertheless, this event is not of much problem, and anyways algorithm converges. However, this event may lead to less smoother convergence. On the other hand, if $|p_t - n_t|$ is very low, then it behaves almost same as the algorithm $A$. Therefore, it is good practice to keep value of $\alpha$ neither very large nor very low. One key thing to discuss, related to learning rate of $A$, as part of BSG. We usually prefer low learning rate to not overshoot the bottom of the valley, however, in BSG its not the matter of concern. Moreover, having high learning rate helps in finding the valley very fast, nevertheless, it may cause less smoother convergence. Major role of learning rate in BSG, is to find the valley, and not more for searching the bottom in the valley. Once valley is found, binary search helps in constraining the parameter update, confined within the defined convex region only.

\section{Experiment}
To investigate the convergence and performance of the proposed optimization algorithm, we empirically evaluate the BSG and compare it with various optimization algorithms. Most of the experimental setup in this section, is on the similar line with the previous papers on stochastic optimization. We run experiments on logistic regression, fully connected neural network and convolutional neural network to evaluate the performance. In experimental setup, used Adam, as part of $A$ and applied this with BSG.

\subsection{Logistic Regression on MNIST}
In this experiment, we use logistic regression, having convex cost surface, on the MNIST handwritten digit dataset. Softmax activation and cross-entropy loss is used. We run the experiment for various optimizers, with default values of their hyperparameters.
\begin{figure}[H]
\centering
\begin{minipage}[H]{.5\textwidth}
  \centering
  \includegraphics[width=1.0\linewidth]{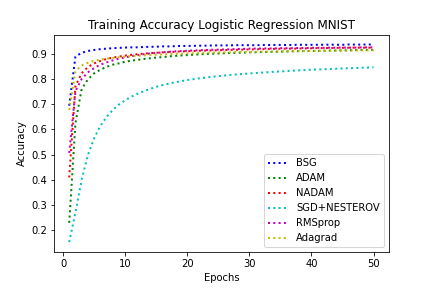}
\caption{Training Accuracy of Logistic Regression on MNIST}
  \label{fig:test1}
\end{minipage}%
\begin{minipage}[H]{.5\textwidth}
  \centering
  \includegraphics[width=1.0\linewidth]{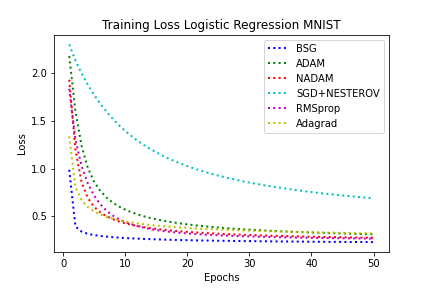}
  \caption{Training Loss of Logistic Regression on MNIST}
  \label{fig:test2}
\end{minipage}
\end{figure}

\begin{figure}[H]
\centering
\begin{minipage}[H]{.5\textwidth}
  \centering
  \includegraphics[width=1.0\linewidth]{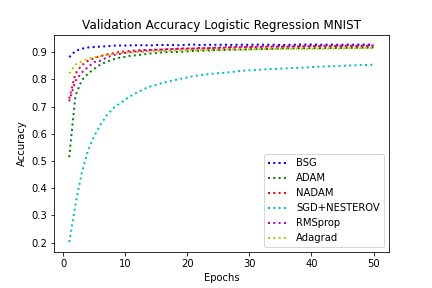}
\caption{Validation Accuracy of Logistic Regression on MNIST}
  \label{fig:test3}
\end{minipage}%
\begin{minipage}[H]{.5\textwidth}
  \centering
  \includegraphics[width=1.0\linewidth]{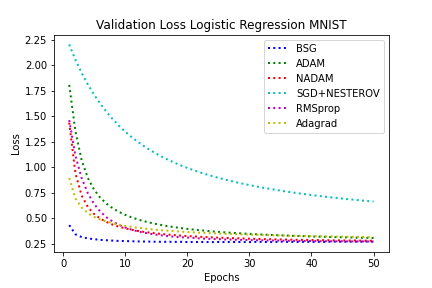}
  \caption{Validation Loss of Logistic Regression on MNIST}
  \label{fig:test4}
\end{minipage}
\end{figure}
As from figure \ref{fig:test1}, it can be seen that, BSG shows state of the art convergence speed. In few iterations, it attains the highest training  accuracy. BSG demonstrates superior convergence speed as well as lowest loss (figure \ref{fig:test2}), as compared to the other optimizers. Figure \ref{fig:test3} and figure \ref{fig:test4} depicts the performance of BSG, in terms of validation accuracy and validation loss respectively. Validation curve follows the same trend as training curve, and achieves better result than other optimizers. Significant difference between convergence speed of BSG and other optimizers is clearly visible in this experiment.

\subsection{Logistic Regression on IMDB}
We use IMDB movie review dataset to evaluate the performance of BSG on sparse data. We keep vocabulary size of 10,000, and represent each data point as bag of words vector. Due to bag of words approach, every vector is very sparse. As data is very sparse, dropout (0.5) is added to handle overfitting. Logistic regression is used with sigmoid activation for binary class classification. In addition, binary-cross-entropy is considered as a loss function. 
\begin{figure}[H]
\centering
\begin{minipage}[H]{.5\textwidth}
  \centering
  \includegraphics[width=1.0\linewidth]{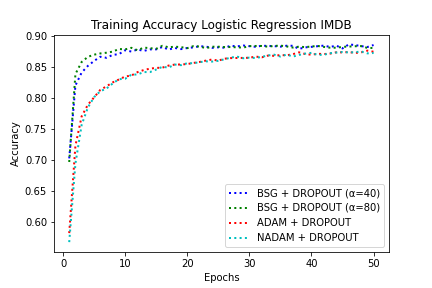}
\caption{Training Accuracy of Logistic Regression on IMDB}
  \label{fig:test5}
\end{minipage}%
\begin{minipage}[H]{.5\textwidth}
  \centering
  \includegraphics[width=1.0\linewidth]{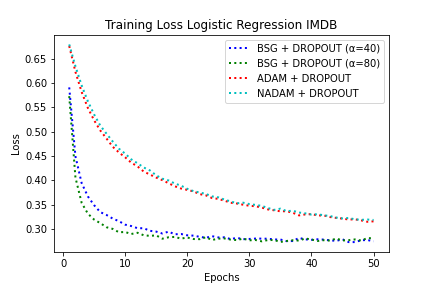}
  \caption{Training Loss of Logistic Regression on IMDB}
  \label{fig:test6}
\end{minipage}
\end{figure}

\begin{figure}[H]
\centering
\begin{minipage}[H]{.5\textwidth}
  \centering
  \includegraphics[width=1.0\linewidth]{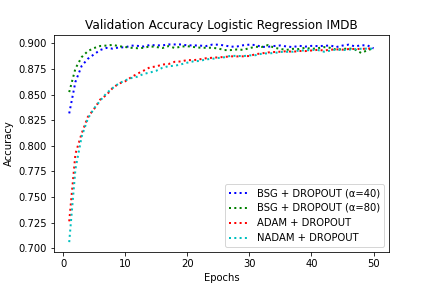}
\caption{validation Accuracy of Logistic Regression on IMDB}
  \label{fig:test7}
\end{minipage}%
\begin{minipage}[H]{.5\textwidth}
  \centering
  \includegraphics[width=1.0\linewidth]{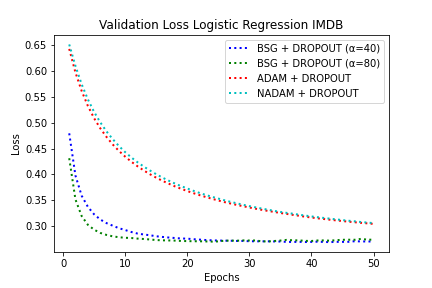}
  \caption{validation Loss of Logistic Regression on IMDB}
  \label{fig:test8}
\end{minipage}
\end{figure}
Training accuracy curve in figure \ref{fig:test5}, and validation accuracy curve in figure \ref{fig:test7}, demonstrates that BSG quickly attains significant accuracy value in less than 10 epochs. The same accuracy value is achieved by other popular algorithms almost after 40 epochs. It shows, the convergence speed of BSG is reasonable better than other methods. One important point to notice regarding interval factor $\alpha$, as the value of $\alpha$ is increased, convergence speed is also increased proportionally.
\subsection{Fully connected neural network on MNIST}
In first two experiments, we evaluated BSG on cost functions, having convex surfaces. In this experimental setup, we use fully connected neural networks (FCNN), using two hidden layers, having 500 and 300 hidden units respectively. \textit{ReLU} activation is applied on both hidden layers. In output layer softmax activation for multi class classification is used and loss function is cross-entropy. Cost function of this network is non-convex in nature. We conduct the experiment by applying batch normalization (BN), and without batch normalization both. 

\begin{figure}[H]
\begin{minipage}[H]{.5\textwidth}
  \includegraphics[width=1.0\linewidth]{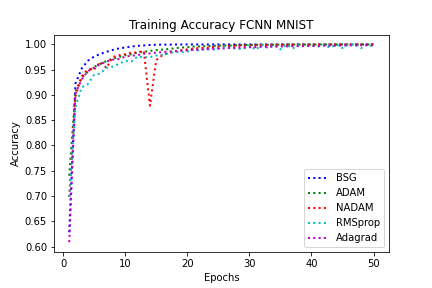}
\caption{Training Accuracy of FCNN on\qquad  MNIST}
  \label{fig:test9}
\end{minipage}%
\begin{minipage}[H]{.5\textwidth}
  \includegraphics[width=1.0\linewidth]{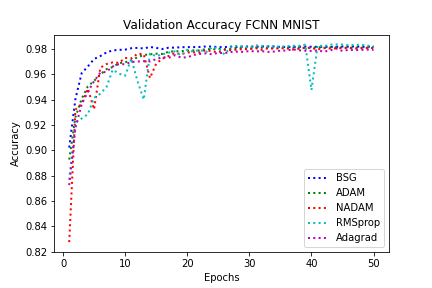}
\caption{Validation Accuracy of FCNN on MNIST}
  \label{fig:test10}
\end{minipage}
\end{figure}

\begin{figure}[H]
\centering
\begin{minipage}[H]{.5\textwidth}
  \centering
  \includegraphics[width=1.0\linewidth]{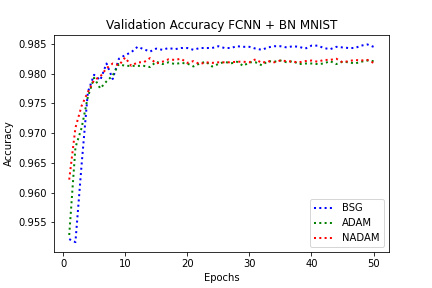}
\caption{Validation Accuracy of FCNN + BN on MNIST}
  \label{fig:test11}
\end{minipage}
\end{figure}
We investigate the performance in both cases. While analysing the performance of BSG on non-convex cost surface, accuracy curves in figure \ref{fig:test9} and in figure \ref{fig:test10} illustrates that, even for the non-convex cost surfaces, BSG apparently dominates the other optimization algorithms, in terms of convergence speed on training data as well as on test data. One observation in figure \ref{fig:test11} is, BSG quickly converges as well as shows, better generalization accuracy as compared with other methods.
\subsection{Convolutional neural network on CIFAR10}
Convolutional neural networks (CNN) are widely used in vision related task. For benchmarking purpose, we apply CNN on the CIFAR10 image dataset. In CNN, two convolutional+maxpooling blocks are used. First convolutional layer consists 64 filters, and (5,5) size kernel while second convolutional layer has 128 filters and (5,5) size kernel. Maxpooling layer of window size (3,3), and stride value as 2 is used in both blocks. CNN blocks are followed by a hidden layer of 300 hidden units. \textit{ReLU} is used as an activation function for all the three hidden layers, whereas, softmax activation for output layer. Cross-entropy is applied as a loss function. Batch normalization method is incorporated in proposed CNN network. While analysing the performance of BSG on the validation set, in figure \ref{fig:test12}, it can be noticed that BSG produces excellent result in terms of accuracy. In validation accuracy plot in figure \ref{fig:test12}, during initial epochs, there is huge gap between the validation curve of BSG and other methods. It reflects the better convergence speed of BSG in terms of validation accuracy. Moreover, BSG converges at the better minima as well, results in better generalization accuracy. Figure \ref{fig:test13}, demonstrates the effect of different interval factor $\alpha$ on the validation curve of BSG. Apparently it can be observed that higher values of $\alpha$ result in higher convergence speed, however validation curve for higher values of $\alpha$ are not that smooth as the validation curve associated with smaller values of $\alpha$ are. 

\begin{figure}[H]
\centering
\begin{minipage}[H]{.5\textwidth}
  \centering
  \includegraphics[width=1.0\linewidth]{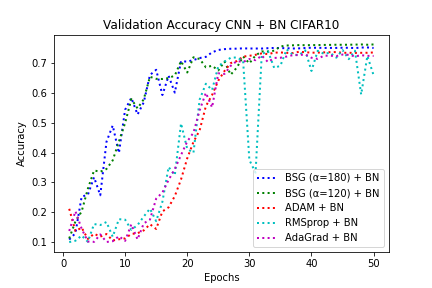}
\caption{Validation Accuracy of CNN + BN CIFAR10}
  \label{fig:test12}
\end{minipage}%
\begin{minipage}[H]{.5\textwidth}
  \centering
  \includegraphics[width=1.0\linewidth]{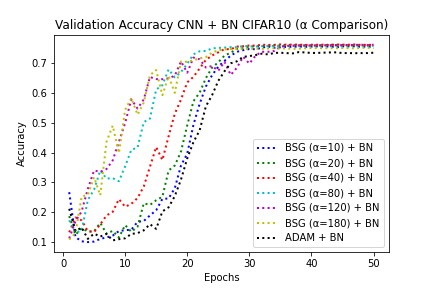}
\caption{Validation Accuracy of CNN + BN CIFAR10 ($\alpha$ Comparison)}
  \label{fig:test13}
\end{minipage}
\end{figure}
\section{CONCLUSION}

In this paper, we introduced a novel algorithm for stochastic optimization, combining binary search as well as first order gradient-based method. Proposed algorithm produces state of the art convergence speed for convex objective functions, however, even for non-convex surfaces, it performs equally well. In experiments, it has been shown that, convergence results produced by BSG, are remarkably well for various classes of dataset. Whether it is high dimensional data, or sparse data, or large data. Implementation of BSG is very straightforward. BSG converges at better minima, resulting reasonable generalization performance. Presented algorithm much depends on the direction of the gradient rather that its magnitude, which helps in overcoming the vanishing and exploding gradient problem in better way. Overall, BSG is well suited to be used in deep leaning as well as classical machine learning, providing fast convergence speed as well as better generalization accuracy.

\bibliographystyle{plain}
\bibliography{ms.bbl}

\begin{thebibliography}{10}

\bibitem{bottou2018optimization}
L{\'e}on Bottou, Frank~E Curtis, and Jorge Nocedal.
\newblock Optimization methods for large-scale machine learning.
\newblock {\em Siam Review}, 60(2):223--311, 2018.

\bibitem{daneshmand2018escaping}
Hadi Daneshmand, Jonas Kohler, Aurelien Lucchi, and Thomas Hofmann.
\newblock Escaping saddles with stochastic gradients.
\newblock {\em arXiv preprint arXiv:1803.05999}, 2018.

\bibitem{daneshm2018escaping}
Hadi Daneshmand, Jonas Kohler, Aurelien Lucchi, and Thomas Hofmann.
\newblock Escaping saddles with stochastic gradients, 2018.

\bibitem{dauphin2014identifying}
Yann Dauphin, Razvan Pascanu, Caglar Gulcehre, Kyunghyun Cho, Surya Ganguli,
  and Yoshua Bengio.
\newblock Identifying and attacking the saddle point problem in
  high-dimensional non-convex optimization, 2014.

\bibitem{Dozat2016IncorporatingNM}
Timothy Dozat.
\newblock Incorporating nesterov momentum into adam.
\newblock In {\em Proceedings of 4th International Conference on Learning
  Representations, Workshop Track}, 2016.

\bibitem{adagrad}
John Duchi, Elad Hazan, and Yoram Singer.
\newblock Adaptive subgradient methods for online learning and stochastic
  optimization.
\newblock {\em J. Mach. Learn. Res.}, 12(null):2121–2159, July 2011.

\bibitem{fang2019sharp}
Cong Fang, Zhouchen Lin, and Tong Zhang.
\newblock Sharp analysis for nonconvex sgd escaping from saddle points, 2019.

\bibitem{ge2015escaping}
Rong Ge, Furong Huang, Chi Jin, and Yang Yuan.
\newblock Escaping from saddle points --- online stochastic gradient for tensor
  decomposition, 2015.

\bibitem{jin2017escape}
Chi Jin, Rong Ge, Praneeth Netrapalli, Sham~M. Kakade, and Michael~I. Jordan.
\newblock How to escape saddle points efficiently, 2017.

\bibitem{kingma2014Adam}
Diederik~P. Kingma and Jimmy Ba.
\newblock Adam: A method for stochastic optimization, 2014.

\bibitem{levy2016power}
Kfir~Y. Levy.
\newblock The power of normalization: Faster evasion of saddle points, 2016.

\bibitem{binary_search}
Anthony Lin et~al.
\newblock Binary search algorithm.
\newblock {\em WikiJournal of Science}, 2(1):1, 2019.

\bibitem{mokhtari2018escaping}
Aryan Mokhtari, Asuman Ozdaglar, and Ali Jadbabaie.
\newblock Escaping saddle points in constrained optimization, 2018.

\bibitem{rmsprop}
Mahesh~Chandra Mukkamala and Matthias Hein.
\newblock Variants of rmsprop and adagrad with logarithmic regret bounds.
\newblock {\em arXiv preprint arXiv:1706.05507}, 2017.

\bibitem{gradient_noise}
Arvind Neelakantan, Luke Vilnis, Quoc~V. Le, Ilya Sutskever, Lukasz Kaiser,
  Karol Kurach, and James Martens.
\newblock Adding gradient noise improves learning for very deep networks, 2015.

\bibitem{momentum_term}
Ning Qian.
\newblock On the momentum term in gradient descent learning algorithms.
\newblock {\em Neural Networks}, 12(1):145–151, Jan 1999.

\bibitem{reddi2019convergence}
Sashank~J. Reddi, Satyen Kale, and Sanjiv Kumar.
\newblock On the convergence of adam and beyond, 2019.

\bibitem{reddi2017generic}
Sashank~J Reddi, Manzil Zaheer, Suvrit Sra, Barnabas Poczos, Francis Bach,
  Ruslan Salakhutdinov, and Alexander~J Smola.
\newblock A generic approach for escaping saddle points, 2017.

\bibitem{robbins1951stochastic}
Herbert Robbins and Sutton Monro.
\newblock A stochastic approximation method.
\newblock {\em The annals of mathematical statistics}, pages 400--407, 1951.

\bibitem{shazeer2018adafactor}
Noam Shazeer and Mitchell Stern.
\newblock Adafactor: Adaptive learning rates with sublinear memory cost, 2018.

\bibitem{sutskever2013importance}
Ilya Sutskever, James Martens, George Dahl, and Geoffrey Hinton.
\newblock On the importance of initialization and momentum in deep learning.
\newblock In {\em International conference on machine learning}, pages
  1139--1147, 2013.

\bibitem{rmsprop_lecture}
Tijmen Tieleman and Geoffrey Hinton.
\newblock Lecture 6.5-rmsprop: Divide the gradient by a running average of its
  recent magnitude.
\newblock {\em COURSERA: Neural networks for machine learning}, 4(2):26--31,
  2012.

\bibitem{ward2018adagrad}
Rachel Ward, Xiaoxia Wu, and Leon Bottou.
\newblock Adagrad stepsizes: Sharp convergence over nonconvex landscapes, from
  any initialization, 2018.

\bibitem{xu2017firstorder}
Yi~Xu, Rong Jin, and Tianbao Yang.
\newblock First-order stochastic algorithms for escaping from saddle points in
  almost linear time, 2017.

\bibitem{zeiler2012adadelta}
Matthew~D. Zeiler.
\newblock Adadelta: An adaptive learning rate method, 2012.

\bibitem{zhou2018convergence}
Dongruo Zhou, Yiqi Tang, Ziyan Yang, Yuan Cao, and Quanquan Gu.
\newblock On the convergence of adaptive gradient methods for nonconvex
  optimization, 2018.

\end{thebibliography}
\end{document}